\documentclass{article}
\usepackage[english]{babel}
\usepackage[letterpaper,top=2cm,bottom=2cm,left=3cm,right=3cm,marginparwidth=1.75cm]{geometry}

\usepackage{authblk}
\usepackage{algorithm,algorithmic,amsthm}

\usepackage{threeparttable}
\usepackage{hyperref}
\usepackage{natbib}

\usepackage[utf8]{inputenc} % allow utf-8 input
\usepackage[T1]{fontenc}    % use 8-bit T1 fonts
\usepackage{hyperref}       % hyperlinks
\usepackage{url}            % simple URL typesetting
\usepackage{booktabs}       % professional-quality tables
\usepackage{amsfonts}       % blackboard math symbols
\usepackage{nicefrac}       % compact symbols for 1/2, etc.
\usepackage{microtype}      % microtypography
\usepackage{algorithm,algorithmic}
\usepackage{xcolor}         % colors
\usepackage{multirow}

\usepackage{amsmath,amscd,amsbsy,amssymb,latexsym,url,bm}
\allowdisplaybreaks
\usepackage{cleveref}
\usepackage{verbatim}
\usepackage{color}
\usepackage{dsfont}

\newtheorem{theorem}{Theorem}

\newtheorem{lemma}{Lemma}

\newtheorem{definition}{Definition}

\newcommand{\set}[1]{\left\{ #1 \right\}}

\newcommand{\cF}{\mathcal{F}}

\newcommand{\cN}{\mathcal{N}}

\newcommand{\cK}{\mathcal{K}}

\newcommand{\true}{\mathrm{True}}
\newcommand{\false}{\mathrm{False}}
\newcommand{\flag}{{\mathrm{F}}}
\newcommand{\ucb}{{\mathrm{UCB}}}
\newcommand{\lcb}{{\mathrm{LCB}}}

\newcommand{\abs}[1]{\left| #1 \right|}
\newcommand{\bOne}[1]{\mathds{1} \! \left\{#1\right\}}
\newcommand{\bracket}[1]{\left(#1\right)}

\newcommand{\EE}[1]{\mathbb{E} \left[#1\right]}
\newcommand{\PP}[1]{\mathbb{P} \left(#1\right)}

\DeclareMathOperator*{\argmax}{argmax}

\mathchardef\mhyphen="2D

\usepackage{hyperref}
\hypersetup{
  colorlinks   = true, %Colours links instead of ugly boxes
  urlcolor     = blue, %Colour for external hyperlinks
  linkcolor    = blue, %Colour of internal links
  citecolor   = blue %Colour of citations
}

\begin{document}

\title{Player-optimal Stable Regret for Bandit Learning in  Matching Markets}

\author[1]{Fang Kong}
\author[1]{Shuai Li \thanks{Corresponding author.}}

\affil[1]{Shanghai Jiao Tong University}
% \affil[2]{Tsinghua University}
% \affil[3]{The Chinese University of Hong Kong, Shenzhen}

\affil[ ]{\{fangkong,shuaili8\}@sjtu.edu.cn}
\date{}

 \maketitle

\begin{abstract} 
The problem of matching markets has been studied for a long time in the literature due to its wide range of applications. Finding a stable matching is a common equilibrium objective in this problem. Since market participants are usually uncertain of their preferences, a rich line of recent works study the online setting where one-side participants (players) learn their unknown preferences from iterative interactions with the other side (arms). Most previous works in this line are only able to derive theoretical guarantees for player-pessimal stable regret, which is defined compared with the players’ least-preferred stable matching. However, under the pessimal stable matching, players only obtain the least reward among all stable matchings. To maximize players’ profits, player-optimal stable matching would be the most desirable. Though \citet{basu21beyond} successfully bring an upper bound for player-optimal stable regret, their result can be exponentially large if players’ preference gap is small. Whether a polynomial guarantee for this regret exists is a significant but still open problem. In this work, we provide a new algorithm named explore-then-Gale-Shapley (ETGS) and show that the optimal stable regret of each player can be upper bounded by $O(K\log T/\Delta^2)$ where $K$ is the number of arms, $T$ is the horizon and $\Delta$ is the players’ minimum preference gap among the first $N+1$-ranked arms. This result significantly improves previous works which either have a weaker player-pessimal stable matching objective or apply only to markets with special assumptions. When the preferences of participants satisfy some special conditions, our regret upper bound also matches the previously derived lower bound.
\end{abstract}

%!TEX root =  main.tex
\section{Introduction}
The model of two-sided matching markets has been extensively studied in the literature \citep{roth1992two} due to its wide range of applications such as labor markets and school admission \citep{roth1984evolution,gale1962college,abdulkadirouglu1999house}.
There are usually two sides of participants, for example, the employers and workers in labor markets. 
Each participant has a preference ranking over the other side based on some utilities such as the working ability of workers.  
Stability characterizes the equilibrium state of the market which guarantees the current matching cannot be easily broken. How to find a stable matching has been studied for a long time \citep{gale1962college,roth1992two}. However, these works simply assume the preference of each participant is known beforehand, which may not be realistic in applications. 
For example, in the online labor market Upwork or online crowd-sourcing platform Amazon Mechanical Turk, the employers are usually uncertain of their preferences over workers since they do not know workers' real working abilities before employment. 
In these online platforms, the employers usually have numerous similar tasks to be delegated and, fortunately, the uncertain preferences can thus be learnt during the iterative matchings with workers through these tasks.

Multi-armed bandits (MAB) is a classic framework that characterizes the learning process of the player in an unknown environment \citep{auer2002finite,lattimore2020bandit}. This framework considers the basic setting consisting of a single player and $K$ arms. The player has unknown preferences over arms. When selecting an arm, this arm would return a random reward and the player can learn its preferences over the arm from these rewards. 
The player's objective is to maximize its cumulative expected rewards over a specified horizon $T$, which is equivalent to minimizing its cumulative expected regret defined as the cumulative reward difference between the most preferred arm and the player's selected arm. The explore-then-commit (ETC) \citep{garivier2016explore}, upper confidence bound (UCB) \citep{auer2002finite} and Thompson sampling (TS) \citep{thompson1933likelihood} are common strategies that achieve sub-linear regret.

Recently, a rich line of works study the bandit learning problem in two-sided matching markets \citep{liu2020competing,liu2021bandit,sankararaman2021dominate,basu21beyond,kong2022thompson,maheshwari2022decentralized,ghosh2022nonstationary}. 
There are several players and arms corresponding to two sides of the market participants. Each player has unknown preferences over arms while arms can be certain of their preferences over players based on some known utilities, for example, the payments of employers in labor markets. 
Since more than one stable matching may exist in the market, there are mainly two types of regrets for players: one is the player-optimal stable regret which is defined compared with the players' most-preferred stable matching; and the other is the weaker player-pessimal stable regret defined compared with the players' least-preferred stable matching. 
\cite{liu2020competing} first study a centralized setting where a central platform assigns allocations of arms to players in each round. When the horizon $T$ and preference gap $\Delta$ are known, they guarantee the player-optimal stable regret for the centralized ETC algorithm.
However, for the classic UCB algorithm which does not require prior knowledge of $T,\Delta$ in this setting, they show that only player-pessimal stable regret can be upper bounded. 
Motivated by real applications where players usually independently make decisions, the following works focus on the more general decentralized setting \citep{liu2021bandit,sankararaman2021dominate,basu21beyond,kong2022thompson,maheshwari2022decentralized,ghosh2022nonstationary}. 
\cite{liu2021bandit} propose a UCB algorithm and \cite{kong2022thompson} propose a TS algorithm for this more general setting, respectively. And again both of them only achieve guarantees for the player-pessimal stable regret.

Under the player-pessimal stable matching, players can only receive the least reward among all stable matchings. To maximize players' profits, the player-optimal stable matching would be the most desirable objective. 
To achieve this goal, \cite{sankararaman2021dominate},  \cite{basu21beyond} and \cite{maheshwari2022decentralized} study some special markets where participants' preferences satisfy some assumptions to ensure unique stable matching, thus the player-pessimal and player-optimal stable regret is equivalent. 
However, these special assumptions may be too strong and not satisfied in real applications. 
For general markets, \cite{basu21beyond} propose an ETC algorithm and successfully derive a theoretical guarantee for the player-optimal stable regret. 
Despite its importance, this regret upper bound has an exponential dependence on $1/\Delta$, which may be huge since the minimum preference gap can be small. 
Whether a polynomial guarantee can be derived for the player-optimal stable regret is a significant but still open problem. 

\subsection{ Our Contributions}\label{sec:contribution}

% \rule{0pt}{12pt}\cite{ghosh2022nonstationary}  & $\displaystyle O\bracket{{K^{\frac{5}{3}}\log^{\frac{4}{3}} T}/{\Delta_D^2}}$ & unique stable matching, non-stationary                  \\\hline
% \fang{
% \begin{remark}
% Definition of $\Delta$. Knowledge on $T$
% \end{remark}
% }

In this work, we make progress on providing guarantees for player-optimal stable regret in general decentralized matching markets. 
We propose the explore-then-Gale-Shapley (ETGS) algorithm inspired by the fundamental idea of exploration-exploitation trade-off for the objective of learning players' preferences and finding the most preferred stable arm. 
Denote $N$ as the number of players, $K$ as the number of arms, $T$ as the horizon, and $\Delta>0$ as the minimum preference gap between the first $N+1$-ranked arms among all players, where we assume $K\ge N$ to ensure each player to have chances of being matched like previous works \citep{liu2020competing,liu2021bandit,sankararaman2021dominate,basu21beyond,kong2022thompson,maheshwari2022decentralized}. Our main result is:

\begin{theorem}\label{thm:main1}(informal)
Following our proposed ETGS algorithm (Algorithm \ref{alg}), the optimal stable regret of each player $p_i$ satisfies
\begin{align*}
    Reg_i(T) = O\bracket{K\log T/\Delta^2}\,.
\end{align*}
\end{theorem}

In Table \ref{table:comparison}, we compare our new upper bound with related results.
Among all previous works, only \cite{liu2020competing} and \cite{basu21beyond} derive the player-optimal stable regret guarantees. 
Though \cite{liu2020competing} achieve the same order of $O(K\log T/\Delta^2)$ for the same type of regret as ours, their algorithm requires strong assumptions. 
Specifically, their algorithm only works under the centralized setting where a central platform coordinates players' behavior in each round and requires $\Delta$ as prior knowledge for hyper-parameters. Since players are uncertain of their preferences, it would be hard to obtain $\Delta$ in real applications. 
Our algorithm works under the general decentralized setting and does not require this prior knowledge. 
\cite{basu21beyond} also study the decentralized setting. However, their guarantee is $\log^{\varepsilon}T$ worse than ours and also has an exponential dependence on $1/\Delta$. When players have similar preferences over arms, i.e., $\Delta$ is small, such a term can blow up. To the best of our knowledge, Theorem \ref{thm:main1} is the first that achieves polynomial player-optimal stable regret for general decentralized matching markets.

Apart from these player-optimal stable regret guarantees, our upper bound also shows a great advantage over previous results which have the weaker objective of player-pessimal stable regret or only work under stronger market assumptions.  
\cite{liu2020competing} also propose another algorithm for the centralized setting that does not require prior knowledge for hyper-parameters. However, for this algorithm, they are only able to bound the player-pessimal stable regret. Compared with this result, our algorithm not only works under a general decentralized setting but also achieves $O(N)$ better regret for the stronger objective of player-optimal stable matching. 
\cite{liu2021bandit} and \cite{kong2022thompson} study the same decentralized setting as ours but again only achieve guarantees for player-pessimal stable regret. 
Our regret upper bound has a significant improvement of $O(N^5K\log T/\varepsilon^{N^4})$ over their results for a much weaker objective. 
\cite{sankararaman2021dominate},  \cite{basu21beyond} and \cite{maheshwari2022decentralized} study the decentralized markets where participants' preferences satisfy some special assumptions to guarantee unique stable matching. 
Compared with these works, we study a more general setting without assumptions on participants' preferences and derive at least $O(N)$ better regret. 

Since our considered general setting covers the special market where all arms have the same preferences studied in \citet{sankararaman2021dominate}, their derived lower bound can also be used to show the optimality of our results. It can be seen that our regret upper bound matches the lower bound in \cite{sankararaman2021dominate} when $N=K$ with a little difference in the definition of $\Delta$.

\begin{table}[htbp]
\centering
\begin{threeparttable}
\caption{Comparisons of settings and regret bounds with most related works. 
$N$ is the number of players, $K\ge N$ is the number of arms, $T$ is the horizon, $\Delta$ is some minimum preference gap, 
$\varepsilon$ depends on the hyper-parameter of algorithms, $C$ is related to the unique stable matching condition and can grow exponentially in $N$, `unique' means that there is only unique stable matching in the market. 
} 
\label{table:comparison}
\begin{tabular}{lll}
\toprule 
                       & Regret bound       & Setting                     \\\hline
\rule{0pt}{13pt}\multirow{2}{*}{\cite{liu2020competing}} & $\displaystyle O\bracket{{K\log T/\Delta^2}}$  & player-optimal, centralized, known $T,\Delta$, gap$_1$\\
& $\displaystyle O\bracket{{NK\log T/\Delta^2}}$              &player-pessimal, centralized, gap$_2$\\\hline
\rule{0pt}{20pt}\cite{liu2021bandit}                                                                                 & $\displaystyle O\bracket{\frac{N^5K^2\log^2 T}{\varepsilon^{N^{4}}\Delta^2}}$                                                       &  player-pessimal, gap$_2$                   \\ \hline
\rule{0pt}{12pt}\multirow{2}{*}{\cite{sankararaman2021dominate}}  & $\displaystyle O\bracket{{NK\log T}/{\Delta^2}}$ & \multirow{2}{*}{unique (serial dictatorship), gap$_1$}                   \\
                       &    $\displaystyle\Omega\bracket{{N\log T}/{\Delta^2}}$                &                          \\\hline
\rule{0pt}{13pt}\multirow{2}{*}{\cite{basu21beyond}} & $\displaystyle O\bracket{K\log^{1+\varepsilon} T+{2^{\bracket{\frac{1}{\Delta^2}}}}^{\frac{1}{\varepsilon}} }$                 & player-optimal, gap$_2$ \\
                       &  $\displaystyle O\bracket{{NK\log T}/{\Delta^2}}$                &   unique (uniqueness consistency), gap$_1$        \\\hline
\rule{0pt}{20pt}\cite{kong2022thompson}                                                                                 & $\displaystyle O\bracket{\frac{N^5K^2\log^2 T}{\varepsilon^{N^{4}}\Delta^2}}$                                                       &  player-pessimal, gap$_2$                   \\ \hline
\rule{0pt}{12pt}\cite{maheshwari2022decentralized}  & $\displaystyle O\bracket{{CNK\log T}/{\Delta^2}}$ & unique ($\alpha$-reducible condition), gap$_1$                  \\
\hline
\rule{0pt}{12pt}Ours&
$\displaystyle O\bracket{{K\log T}/{\Delta^2}}$ &    player-optimal, known $T$, gap$_3$                  \\
          \bottomrule 
\end{tabular}
\begin{tablenotes}
\item[1] The definition of $\Delta$ in different works requires particular care. We use gap$_1$, gap$_2$, gap$_3$ represent the minimum preference gap between the (player-optimal) stable arm and the next arm after the stable arm in the preference ranking among all players, the minimum preference gap between any different arms among all players, and the minimum preference gap between the first $N+1$ ranked arms among all players, respectively. Based on the property of the Gale-Shapley algorithm provided in Lemma \ref{lem:GS:complexity} which shows that the player-optimal stable arm must be the first $N$-ranked, there would be gap$_1 >$gap$_3 > $ gap$_2$.  
\end{tablenotes}
\end{threeparttable}
\end{table}

%!TEX root =  main.tex

\section{Related Work}\label{sec:related}

Multi-armed bandits (MAB) have been studied for a long time \citep{thompson1933likelihood,robbins1952some}. Please see \cite{lattimore2020bandit} for a detailed survey. 
% In this paper, we only introduce the most related stochastic MAB problem. In stochastic bandits, 
In this problem, each arm has an unknown but fixed reward distribution. Each time the player selects an arm, this arm would produce a random reward sampled from its corresponding distribution. 
How to balance the exploration and exploitation is the key in the learning process. 
There are mainly three types of algorithms to solve this problem. The first is the ETC type \citep{garivier2016explore}, which uniformly explores arms to collect observations for a period and then focuses on the arm with the best performances. The second type is optimism in the face of uncertainty (OFU), such as UCB \citep{auer2002finite}, which establishes confidence set for each arm and selects the arm with the largest upper confidence bound in each round. The third type is TS \citep{thompson1933likelihood,agrawal2013further}, which is based on the Bayesian approach and aims at arms with a high probability of being optimal. All these three types of algorithms are shown to achieve $O(\log T/\Delta)$ regret where $T$ is the horizon and $\Delta$ is the minimum expected reward gap.

\cite{das2005two} first introduce the bandit problem in matching markets. They consider a special setting where all players and arms share the same preferences and propose algorithms with experimental verification. 
Later, \cite{liu2020competing} consider a variant of the problem where one-side preferences are unknown in general markets and successfully derive the theoretical guarantees. 
They propose both ETC and UCB-type algorithms, the former which requires the knowledge of $T$ and $\Delta$ achieves sub-linear optimal stable regret,  while the latter only achieves the pessimal stable regret guarantees. For the UCB algorithm, they also derive a counterexample and show it cannot achieve sub-linear player-optimal stable regret. 

Note both algorithms in \cite{liu2020competing} are based on a centralized setting where a central platform collects participants' preferences and assigns allocations to them. 
Since players usually independently make decisions in real applications, the following works all focus on the decentralized setting. 
\cite{liu2021bandit} and \cite{kong2022thompson} give the first UCB and TS-type algorithm for general decentralized market respectively and both achieve $O(\log^2 T)$ stable regret.
To improve these results, \cite{sankararaman2021dominate},  \cite{basu21beyond} and \cite{maheshwari2022decentralized} successively study some special markets where only unique stable matching exists and  achieve $O(\log T)$ stable regret.

All these decentralized-based results only achieve pessimal stable regret guarantees (when the stable matching is unique, the player-optimal and pessimal stable regret are equivalent). 
Since players can obtain the most rewards in the optimal stable matching, the optimal stable regret would be more desirable. \cite{basu21beyond} try this possibility and propose an ETC-type algorithm. Their optimal stable regret has sub-linear $O(\log^{1+\varepsilon}T)$ dependence on $T$ but suffers from exponential constant term $O({2^{\bracket{{1}/{\Delta^2}}}}^{{1}/{\varepsilon}} )$ with respect to $1/\Delta$, where $\varepsilon$ is a hyper-parameter. 

There are also other works related to the learning problem in matching markets.
\cite{ghosh2022nonstationary} study the non-stationary environments where players' preferences can slowly vary but all arms have the same preferences. 
\cite{jagadeesan2021learning} study the case with transferable utilities where the matching is accompanied by monetary transfers. Their objective is also different from ours which is defined with respect to the subset instability. 
\cite{min2022learn} study the Markov matching market by considering unknown state transitions when matching occurs. 
\cite{dai2020learning} and \cite{dai2021learning} also study to learn participants' preferences but their settings involve no cumulative regret.

%!TEX root =  main.tex
\section{Preliminaries}\label{sec:setting}

In this section, we formally introduce the problem setting of bandit learning in matching markets. 
For a positive integer $n$, we will use $[n]$ to represent $\set{1,2,\ldots, n}$.

Denote $N$ and $K$ as the number of players and arms in the market.  
Let $\cN = \set{p_1,p_2,\ldots,p_N}$ be the player set and $\cK=\set{a_1,a_2,\ldots,a_K}$ be the arm set. 
As previous works, to ensure each player to have chances of being matched, we assume $N\le K$  \citep{liu2020competing,liu2021bandit,sankararaman2021dominate,basu21beyond,kong2022thompson,maheshwari2022decentralized}. 

The preference of player $p_i$ over arm $a_j$ can be portrayed by a real value $\mu_{i,j}\in[0,1]$, where the larger value of $\mu_{i,j}$ implies more preference on arm $a_j$. 
Without loss of generality, we assume all preferences are distinct, i.e., $\mu_{i,j}\neq \mu_{i,j'}$ for any different arms $a_j\neq a_{j'}$, similar to previous works \citep{roth1992two,liu2020competing,liu2021bandit,sankararaman2021dominate,basu21beyond,kong2022thompson,maheshwari2022decentralized}. Arms also have a preference ranking over players. Denote $(\pi_{j,i})_{i\in[N]}$ as the distinct preference values of arm $a_j$ over players. Then $\pi_{j,i}>\pi_{j,i'}$ implies $a_j$ prefers $p_i$ to $p_{i'}$. 
Motivated by real applications such as online labor market Upwork, the preferences of players are usually uncertain and can be learnt through iterative matching processes. 
While arms usually know their preferences based on some known utilities such as the payment of employers.

In each round $t=1,2,\ldots$, each player $p_i$ proposes to an arm $A_i(t)\in \cK$.
Correspondingly, each arm $a_j$ receives applications from players in $A_j^{-1}(t):= \set{p_i: A_i(t) = a_j }$. 
Similar to the labor market where a worker can only work for one task, arms would only accept one application from the player that it prefers most. Denote $\bar{A}^{-1}_j(t)$ as the successfully matched player of arm $a_j$. That is to say,  $\bar{A}^{-1}_j(t) \in \argmax_{p_i\in A_j^{-1}(t)}\pi_{j,i}$. 
Similarly, let $\bar{A}_i(t)$ represent the successfully matched arm of player $p_i$. Then $\bar{A}_i(t)=A_i(t)$ if $p_i$ is successfully accepted by $A_i(t)$. Otherwise if $p_i$ is rejected, we simply let $\bar{A}_i(t)=\emptyset$. 
Once $p_i$ is successfully accepted by the proposed arm, it will receive a random reward $X_i(t)$ characterizing its matching experience in this round, which we assume is $1$-subgaussian with expectation $\mu_{i,\bar{A}_i(t)}$. 
And if $p_i$ is rejected, it only receives $X_i(t)=0$.  
For convenience, denote $A(t)=\set{(i,A_i(t)):{i\in[N]}}$ as the selections of all players and $\bar{A}(t)=\set{(i,\bar{A}_i(t)):{i\in[N]}}$ as the final matching at round $t$. 

Stability is a key concept to characterize the matching in the market. Formally, the matching $\bar{A}(t)$ is stable if no player and arm has incentive to abandon their current partner, i.e., there exists no player-arm pair $(p_i,a_j)$ such that $\mu_{i,j}>\mu_{i,\bar{A}_i(t)}$ and $\pi_{j,i}>\pi_{j,\bar{A}_j^{-1}(t)}$, where we simply define $\pi_{j,\emptyset} = -\infty$ and $\mu_{i,\emptyset} = -\infty$ for each $j\in[K],i\in[N]$. 
There may be more than one stable matching in the market. 
% In this work, we aim to learn the player-optimal stable matching and minimize the player-optimal stable regret for each $p_i\in\cN$. 
% Specifically, 
Let $M:=\set{ m: m \text{ is a stable matching}}$ be the set of all stable matchings and $m^* = \set{(i,m^*_i):{i\in[N]}}\in M$ be the players' most preferred one. That is to say, $\mu_{i,m^*_i}\geq\mu_{i,m_i}$ for any $m\in M, i\in[N]$. 
Our objective is to learn the player-optimal stable matching $m^*$ and minimize the player-optimal stable regret for each $p_i\in\cN$, which is defined as the cumulative reward difference between being matched with $m^*_i$ and that $p_i$ receives over $T$ rounds: 
\begin{align}
    Reg_i(T) = \sum_{t=1}^T \mu_{i,m^*_i} - \EE{\sum_{t=1}^T X_{i}(t)}\,.
\end{align}
Here the expectation is taken over the randomness of the received reward and the players' strategy.

% To be self-contained, we also give a proof for the optimality of the stable matching returned by GS in Appendix. 

\paragraph{Importance of Player-optimal Stable Regret} 
Denote $\underline{m} = \set{(i,\underline{m}_i):{i\in[N]}}$ as the player-pessimal stable matching, i.e., $\mu_{i,m_i} \geq \mu_{i,\underline{m}_i}$ for any $m\in M, i\in[N]$.
% $\underline{m}_i$ is the least-preferred arm of player $p_i$ among all arms that are matched with $p_i$ in stable matchings. 
Correspondingly, the player-pessimal stable regret $\underline{Reg}_i(T)$ is defined as the cumulative difference between the reward of the pessimal stable arm and the real received rewards. 
From the definition, it is straightforward that
\begin{align*}
    \underline{Reg}_i(T) := \sum_{t=1}^T \mu_{i,\underline{m}_i} - \EE{\sum_{t=1}^T X_{i}(t)} \le \sum_{t=1}^T \mu_{i,m^*_i} - \EE{\sum_{t=1}^T X_{i}(t)}:= Reg_i(T) 
\end{align*}
for any player $p_i \in \cN$, where the equations hold according to the regret definition and the inequality is due to the definition of $\underline{m}$ and $m^*$. 
Above all, an upper bound for the player-optimal stable regret is also an upper bound for the player-pessimal stable regret, but not vice versa, meaning that the player-optimal stable regret is a much stronger objective. 
And since there is a constant difference between $\mu_{i,{m}^*_i}$ and $\mu_{i,\underline{m}_i}$ when more than one stable matching exist, the difference between player-optimal stable regret and pessimal stable regret can be $O(T)$.
% in this case.
% These imply that our objective for player-optimal stable regret is stronger than previous work for pessimal stable regret. 
% Thus the player-optimal stable regret of an algorithm which suffer sub-linear pessimal stable regret may grow linear in $T$. 

\paragraph{Offline Gale-Shapley Algorithm.}
We call the case where all players know their exact preferences as the \textit{offline} setting.  In this setting, the Gale-Shapley (GS) algorithm \citep{gale1962college} is a well-known algorithm to find the player-optimal stable matching.
The GS algorithm proceeds in several steps. 
In the first step, each player independently proposes to the arm it prefers most. Each arm will reject all but the player it prefers most among those who propose to it. Then at the following each step, players still propose to their most preferred arm among those who have not rejected it and arms would only accept its most preferred one among players who propose to it. When no rejection happens, the algorithm stops and returns the last matching. 
It has been shown that the final matching is just the player-optimal stable matching $m^*$ \citep[Theorem 1 and Theorem 2]{gale1962college}. 
We also show that the player-optimal stable arm for each player must be its first $N$-ranked and the GS algorithm will stop in at most $N^2$ steps in Lemma \ref{lem:GS:complexity}.

\section{Algorithm}\label{sec:algo}

In this section, we present our explore-then-Gale-Shapley algorithm (ETGS, Algorithm \ref{alg}) for players in decentralized matching markets. 
The algorithm consists of three phases. 
During the first phase,
each player aims to learn a distinct index (Line \ref{alg:p1:start}-\ref{alg:p1:end}). 
And in the second phase (Line \ref{alg:p2:start}-\ref{alg:p2:end}), players would sufficiently explore different arms with the index obtained in the first phase and aim to get an accurate estimation of their preference ranking over the first $N$-ranked arms. 
With these estimated preferences, they hope to find the arm in the optimal stable matching in the third phase and would focus on this arm in all of the following rounds (Line \ref{alg:p3:start}-\ref{alg:p3:end}). 

\begin{algorithm}[thb!]
    \caption{explore-then-Gale-Shapley (ETGS, from view of player $p_i$)}\label{alg}
    \begin{algorithmic}[1]
    \STATE Input: player set $\cN$, arm set $\cK$, horizon $T$\label{alg:input}
    \STATE Initialize: $\hat{\mu}_{i,j}=0, T_{i,j}=0, \forall j\in[K]$
    \STATE //Phase 1, index estimation \label{alg:p1:start} 
    \STATE Arm = $a_1$
    \FOR{round $t=1,2,\ldots,N$}
        \STATE $A_i(t)=$Arm \label{alg:p1:select}
        \IF{$\bar{A}_i(t) = A_i(t) =a_1$}
            \STATE Index $=t $; Arm = $a_2$ \label{alg:p1:getIndex}
        \ENDIF
    \ENDFOR \label{alg:p1:end} 
    \STATE //Phase 2, learn the preferences\label{alg:p2:start} 
    \FOR{$\ell=1,2,\ldots $}
        \STATE $\flag_{\ell} = \false$ //whether the preference has been estimated well
        \FOR{$t=N+\sum_{\ell'=1}^{\ell-1}(2^{\ell'}+1)+1,\ldots, N+\sum_{\ell'=1}^{\ell-1}(2^{\ell'}+1)+2^{\ell}$} \label{alg:p2:learn:start} 
            % \STATE $t = N+\sum_{\ell'=1}^{\ell-1}(2^{\ell'}+1)+ t'$
            \STATE  $A_i(t)=a_{(\text{Index}+t-1)\%K+1}$ \label{alg:p2:learn:select}
            \STATE Observe $X_{i,A_i(t)}(t)$ and update $\hat{\mu}_{i,A_i(t)}$, $T_{i,A_i(t)}$ if $\bar{A}_i(t) = A_i(t) $ \label{alg:p2:learn:update}
        \ENDFOR\label{alg:p2:learn:end} 
        \STATE Compute $\ucb_{i,j}$ and $\lcb_{i,j}$ for each $j\in[K]$ \label{alg:p2:computeUCBLCB}
        \IF{$\exists \sigma$ such that $\lcb_{i,\sigma_{k}}>\ucb_{i,\sigma_{k+1}}$ for any $k\in[N]$ and $\lcb_{i,\sigma_{N}}>\ucb_{i,\sigma_{k}}$ for any $k=N+1,N+2,...,K$}\label{alg:p2:learnedRanking}
        \STATE $\flag_{\ell} = \true$ and  $\sigma_{i} = \sigma$
        \ENDIF
        \STATE Initialize $O_{\ell}=\emptyset$\label{alg:p2:initilizeO} 
            \STATE $t = N+\sum_{\ell'=1}^{\ell-1}(2^{\ell'}+1)+2^{\ell}+ 1$\label{alg:monitor:round}
            \STATE $A_i(t) = a_{\text{Index}}$ if $\flag_{\ell}==\true$ and $A_i(t)=\emptyset$ otherwise \label{alg:p2:monitor:pullInOwnBlock}
            \STATE Update $O_{\ell}=\cup_{i'\in[N]} \set{\bar{A}_{i'}(t)}$ \label{alg:p2:monitor:updateObservation}
            \IF{$\abs{O_{\ell}} == N$} \label{alg:p2:monitor:detectP3}
                \STATE Enter in Phase 3 with $\sigma_{i}$; $t_2 = t$ //$t_2$ is the round when phase 2 ends  \label{alg:p2:enterP3}
            \ENDIF
        % \FOR{$t'=1,2,\ldots, NK$} \label{alg:p2:monitor:start} 
        %     \STATE $t = N+\sum_{\ell'=1}^{\ell-1}(2^{\ell'}+NK)+2^{\ell}+ t'$
        %     \IF{$K\cdot(\text{Index}-1) <t' \le K\cdot \text{Index}$} \label{alg:p2:monitor:ownBlock}
        %     \STATE $A_i(t) = \sigma_{i,t\%K+1}$ if $\flag_{\ell}==\true$ and $A_i(t) \in\argmax_{j\in[K]} \ucb_{i,j}$ otherwise \label{alg:p2:monitor:pullInOwnBlock}
        %     \ELSE
        %         \STATE $A_i(t)=\emptyset$\label{alg:p2:monitor:pullInOtherBlock}
        %     \ENDIF
        %     \STATE Update $O_{\ell,i'}=O_{\ell,i'}\cup \set{\bar{A}_{i'}(t)}$ for each $i'\in[N]$ \label{alg:p2:monitor:updateObservation}
        %     \IF{$t'== NK$ and $O_{\ell,i'} == \cK$ for each $i'\in[N]$} \label{alg:p2:monitor:detectP3}
        %         \STATE Enter in Phase 3 with $\sigma_{i}$; $t_2 = t$ \label{alg:p2:enterP3}
        %     \ENDIF
        % \ENDFOR\label{alg:p2:monitor:end} 
    \ENDFOR\label{alg:p2:end} 
    \STATE //Phase 3, find the optimal stable arm with $\sigma_i = (\sigma_{i,1}, \sigma_{i,2},\ldots, \sigma_{i,K})$\label{alg:p3:start} 
    \STATE Initialize $s=1$
    \FOR{$t=t_2+1,t_2+2,\ldots$}
        % \STATE $t=t_2+t'$
        \STATE $A_{i}(t) = a_{\sigma_{i,s}}$
            \STATE $s=s+1$ if $\bar{A}_i(t) == \emptyset$
    \ENDFOR\label{alg:p3:end} 
    \end{algorithmic}
\end{algorithm}

The first phase proceeds in $N$ rounds (Line \ref{alg:p1:start}-\ref{alg:p1:end}).
At the first round $t=1$, all players would propose to arm $a_1$ (Line \ref{alg:p1:select}) and only the player who is successfully accepted gets the index $1$ (Line \ref{alg:p1:getIndex}). In the second round, all of other players (except for the player who gets index $1$) still propose to $a_1$ (Line \ref{alg:p1:select}) and the only accepted player gets the index $2$ (Line \ref{alg:p1:getIndex}).  Similar actions would be taken in the following round $3,4,\ldots,N$.
Intuitively, the index of each player $p_i$ is just the order of $p_i$ in the preference ranking of $a_1$. 
At the end of this phase, each player can obtain a distinct index. 

After $N$ rounds of estimating indices, the algorithm enters in the second phase to explore arms. 
This phase can be divided into several sub-phases $\ell=1,2,\ldots$ with each sub-phase $\ell$ consisting of $2^{\ell}+1$ rounds  (Line \ref{alg:p2:start}-\ref{alg:p2:end}). 
Each sub-phase $\ell$ starts with an exploration stage of length $2^\ell$ (Line \ref{alg:p2:learn:start}-\ref{alg:p2:learn:end}) and ends with a monitoring round of length $1$ (Line \ref{alg:monitor:round}). 
At a high-level, players hope to collect as many observations on arms during the exploration stage and detect whether the preferences values have been learnt well during the monitoring round. 
When players detect that the preferences have been estimated well in the monitoring round, they will end up the exploration and enter in the third phase for finding arms in the optimal stable matching (Line \ref{alg:p2:enterP3}). 
% The length of the exploration stage grows exponentially with $2^\ell$ rounds at stage $\ell$ (Line \ref{alg:p2:learn:start}) and the length of the monitoring stage is fixed with $NK$ rounds during all sub-phases (Line \ref{alg:p2:monitor:start}).

Specifically, during the exploration stage of each sub-phase, players propose to arms in a round-robin way (Line \ref{alg:p2:learn:select}). 
Based on their distinct indices obtained in the first phase, different players can ensure to select different arms at each round and each of them can thus be successfully accepted. 
Once $p_i$ gets an observation on the selected arm $A_i(t)$, it would update the estimated preference value $\hat{\mu}_{i,A_i(t)}$ and the observed time $T_{i,A_i(t)}$ for arm $A_i(t)$ (Line \ref{alg:p2:learn:update}) as
\begin{align*}
    \hat{\mu}_{i,A_i(t)} = \bracket{\hat{\mu}_{i,A_i(t)}\cdot T_{i,A_i(t)} + X_{i,A_i(t)}(t) }/{(T_{i,A_i(t)}+1)}\,,\ T_{i,A_i(t)} = T_{i,A_i(t)}+1 \,.
\end{align*}
At the end of the exploration stage, players would construct a confidence set for each estimated preference value based on its previously collected observations. 
Specifically, the confidence set of $p_i$ for the preference value over $a_j$ is established with the upper bound $\ucb$ and lower bound $\lcb$ defined as 
\begin{align}
\ucb_{i,j} = \hat{\mu}_{i,j}+\sqrt{{6\log T}/{T_{i,j}}}\,, ~~~
 \lcb_{i,j} = \hat{\mu}_{i,j}-\sqrt{{6\log T}/{T_{i,j}}}\,,  \label{eq:def:UCBLCB} 
\end{align}
where we simply let $\ucb_{i,j}$ be $\infty$ and $\lcb_{i,j}$ be $-\infty$ when $T_{i,j}=0$. 
Once the confidence sets for two arms $a_j,a_{j'}$ are disjoint, i.e., $\lcb_{i,j}>\ucb_{i,j'}$ or $\lcb_{i,j'}>\ucb_{i,j}$, $p_i$ can determine its preference over these arms.  
When $p_i$ determines the first $N$-ranked arms in its preferences ranking as Line \ref{alg:p2:learnedRanking}, $p_i$ marks the flag $\flag_{\ell}$ in this sub-phase as $\true$ and marks the estimated preference ranking $\sigma_i$ as this permutation $\sigma$.

% When the confidence sets for all arms are disjoint, i.e., there exists a permutation $\sigma$ over arms such that $\lcb_{i,\sigma_{k}}>\ucb_{i,\sigma_{k+1}}$ for each $k\in[K-1]$, $p_i$ can believe that it gets an accurate preference estimation.

The monitoring round is established for players to detect whether all players have estimated an accurate preference ranking. 
% Recall that this stage contains $NK$ rounds, which can be further divided into $N$ blocks with each composed of $K$ rounds. 
% Then according to players' indices got in the first phase, each player can correspond to one block. 
If $p_i$ has already estimated an accurate preference ranking, i.e., $\flag_{\ell}=\true$, then it would propose to the arm with its distinct index. Otherwise, it gives up the chance of selecting arms (Line \ref{alg:p2:monitor:pullInOwnBlock}). 
% only selects a fixed arm with the largest upper confidence bound (Line \ref{alg:p2:monitor:ownBlock}-\ref{alg:p2:monitor:pullInOwnBlock}). 
% In blocks corresponding to other players, $p_i$ would give up the chance of selecting arms (Line \ref{alg:p2:monitor:pullInOtherBlock}). 
Motivated by real applications such as online labor market Upwork where workers usually update their working experience in the profile, we also assume all players can observe the successfully accepted player for each arm at the end of the round as previous works \citep{liu2021bandit,kong2022thompson,ghosh2022nonstationary}. 
Thus players can maintain a set $O_{\ell}$ to record arms which have been successfully matched with players at sub-phase $\ell$ (Line \ref{alg:p2:monitor:updateObservation}). Then at the end of the monitoring round, if $p_i$ observes that all players have successfully matched an arm, i.e., $\abs{O_{\ell}}=N$ (Line \ref{alg:p2:monitor:detectP3}), $p_i$ can believe all players have already estimated an accurate preference ranking and would enter in the third phase with the estimated ranking $\sigma_i$ (Line \ref{alg:p2:enterP3}).

At a high level, the third phase (Line \ref{alg:p3:start}-\ref{alg:p3:end}) can be regarded as a decentralized Gale-Shapley algorithm \citep{gale1962college}. 
During this phase, players aim to find and focus on the arm in the optimal stable matching with the estimated ranking $\sigma_i$. 
Specifically, $p_i$ would propose to arms one by one according to $\sigma_i$ until no rejection happens.
When each player $p_i$'s estimated ranking $\sigma_i$ for the first $N$ arms is accurate, this procedure is expected to find the real optimal stable arm for each player. 

%!TEX root =  main.tex

\section{Theoretical Analysis}\label{sec:results}

Before giving the formal regret guarantee for Algorithm \ref{alg}, we first introduce some useful notations. The following Definition \ref{def:gaps} defines corresponding preference gaps to measure the hardness of the learning problem to reach the player-optimal stable matching.

\begin{definition}\label{def:gaps}
For each player $p_i$ and arm $a_j\neq a_{j'}$, let $\Delta_{i,j,j'} = \abs{\mu_{i,j}-\mu_{i,j'}}$ be the preference gap of $p_i$ between $a_j$ and $a_j'$. 
Let $\rho_{i}$ be player $p_i$'s preference ranking and $\rho_{i,k}$ be the $k$-th preferred arm in $p_i$'s ranking for $k\in[K]$. 
Define $\Delta = \min_{i\in[N];k,k'\in[N+1];k\neq k'}\Delta_{i,\rho_{i,k},\rho_{i,k'}}$ as the minimum preference gap among all players and their first $N+1$-ranked arms, which is non-negative since all preferences are distinct. 
Further, for each player $p_i$, let $\Delta_{i,\max} = \mu_{i,m^*_i}$ be the maximum player-optimal stable regret that may be suffered by $p_i$ in all rounds. 
% and arm $a_j$, define $\Delta_{i,j} = \max\set{\mu_{i,m^*_i}, \mu_{i,m^*_i} - \mu_{i,j}} =\mu_{i,m^*_i} $ as the player-optimal stable regret that needs to be paid by $p_i$ when proposing to $a_j$. Let $\Delta_{i,\max} = \max_{j\in[K]}\Delta_{i,j}$ be the maximum player-optimal stable regret that may be suffered by $p_i$ in all rounds. 
\end{definition}

We now present the upper bound for the player-optimal stable regret of each player by following the ETGS algorithm. 

\begin{theorem}\label{thm:main}(Restate of Theorem \ref{thm:main1})
Following the ETGS algorithm (Algorithm \ref{alg}), the optimal stable regret of each player $p_i\in\cN$ satisfies
\begin{align}
    Reg_i(T) &\le \bracket{N+ \frac{192K\log T}{\Delta^2}+\log\bracket{\frac{192K\log T}{\Delta^2}} +N^2+ 2NK}\cdot \Delta_{i,\max} \label{eq:regret:full} \\
    &= O(K\log T/\Delta^2)\,.
\end{align}

\end{theorem}

At a high level, the first term in Eq. \eqref{eq:regret:full} is the upper bound for regret incurred in phase $1$, the second term is the regret upper bound for the total exploration rounds and the third term is the upper bound for the total monitoring rounds in phase $2$, the fourth term is the regret upper bound for phase $3$ since GS finds the player-optimal stable matching in at most $N^2$ rounds with the correct preference ranking for the top $N$ arms, and the last constant term corresponds to the bad concentration events. 
The detailed proof of Theorem \ref{thm:main} is deferred to Section \ref{subsec:proof}.  

Please see Section \ref{sec:contribution} for a detailed comparison and discussion of our result with previous related works. 
Intuitively speaking, previous works \citep{liu2020competing,liu2021bandit,kong2022thompson} fail to efficiently learn the player-optimal stable matching through interactions may be due to the inappropriate way to utilize the ideas behind classical bandit algorithms. A core in the design of bandit algorithms is to balance the exploration-exploitation (EE) trade-off and the classical UCB and TS algorithms give efficient ways to achieve this goal. However, directly applying these classical bandit algorithms to matching markets may not work. As discussed in \cite{liu2020competing}, adopting UCB estimations followed by offline GS algorithm might not learn the optimal stable matching since the UCB estimation of a player-arm pair, though optimistic, does not guarantee exploration.
% The setting of matching markets is quite different from standard bandit settings with the influence of other market participants. 
Our algorithmic design is inspired by the fundamental idea of EE trade-off. We guarantee the explorations to estimate an accurate preference ranking using manually designed players' selections and avoid additional regret by only keeping necessary explorations. 
Then we find that the exploitation goes through by following the scheduling of the offline GS when enough explorations are guaranteed for a correct preference ranking. With such a more appropriate EE trade-off, we could derive a much better result compared with previous works.

\subsection{Proof of Theorem \ref{thm:main}.}\label{subsec:proof}
For convenience, let $\hat{\mu}_{i,j}(t), T_{i,j}(t), \ucb_{i,j}(t), \lcb_{i,j}(t)$ be the value of $\hat{\mu}_{i,j}, T_{i,j}, \ucb_{i,j}$ and $\lcb_{i,j}$ at the end of round $t$, respectively. 
Define $\cF= \set{\exists t\in[T], i\in[N],j\in[K]: \abs{\hat{\mu}_{i,j}(t)-\mu_{i,j}} > \sqrt{\frac{6\log T}{T_{i,j}(t)}}  }$ as the bad event that some preference is not estimated well during the horizon. 
Since all players have the same observations, we can conclude that all of them would enter the third phase simultaneously according to Line \ref{alg:p2:monitor:detectP3} in Algorithm \ref{alg}. 
Denote $\ell_{\max}$ as the largest sub-phase number of phase 2. That is to say, players enter in phase 3 at the end of sub-phase $\ell_{\max}$. We then provide the proof of Theorem \ref{thm:main} as follows. 

\begin{proof}[Proof of Theorem \ref{thm:main}]
The optimal stable regret of each player $p_i$ by following Algorithm \ref{alg} satisfies
\begin{align}
    Reg_i(T) &=\EE{\sum_{t=1}^T \bracket{\mu_{i,m^*_i} - X_i(t)} } \notag \\
    &\le \EE{\sum_{t=1}^T \bOne{\bar{A}(t) \neq m^*} \cdot \Delta_{i,\max}} \notag \\
    % &\le N \Delta_{i,\max}+ \EE{\sum_{t=N+1}^T \bOne{{A}(t) \neq m^*} }\cdot \Delta_{i,\max} \label{eq:dueto:nocollison} \\
    &\le N \Delta_{i,\max}+ \EE{\sum_{t=N+1}^T \bOne{\bar{A}(t) \neq m^*} \mid \urcorner \cF }\cdot \Delta_{i,\max} + \PP{\cF} \cdot T \cdot  \Delta_{i,\max} \notag \\
    &\le N \Delta_{i,\max}+ \EE{\sum_{t=N+1}^T \bOne{\bar{A}(t) \neq m^*} \mid \urcorner \cF }\cdot \Delta_{i,\max}  + 2NK\Delta_{i,\max} \label{eq:dueto:bad} \\
    &\le N \Delta_{i,\max}+\EE{ \sum_{\ell=1}^{\ell_{\max}  }\bracket{2^{\ell} +1} +N^2 \mid \urcorner \cF }\cdot \Delta_{\max} + 2NK\Delta_{i,\max} \label{eq:dueto:phases:and:def_lmax}\\
    &\le {N \Delta_{i,\max}}+  {\bracket{ \frac{192K\log T}{\Delta^2} + \log\bracket{ \frac{192K\log T}{\Delta^2}} }\cdot \Delta_{i,\max}} +{N^2\Delta_{i,\max}} + 2NK\Delta_{i,\max} \label{eq:end} \,,
    % &\le \underbrace{N \Delta_{i,\max}}_{\text{phase 1}}+  \underbrace{\bracket{ \frac{64K\log T}{\Delta^2} + NK\log\bracket{ \frac{64K\log T}{\Delta^2}} }\cdot \Delta_{i,\max}}_{\text{phase 2}} +\underbrace{K^2\Delta_{i,\max}}_{\text{phase 3}} + \underbrace{\frac{NK\pi^2}{3}\Delta_{i,\max}}_{\text{bad event}} \label{eq:end} \,,
\end{align}
where Eq.\eqref{eq:dueto:bad} comes from Lemma \ref{lem:badevent}, Eq. \eqref{eq:dueto:phases:and:def_lmax} holds according to Algorithm \ref{alg} and Lemma \ref{lem:phase3}, Eq. \eqref{eq:end} holds based on Lemma \ref{lem:phase2}. 
\end{proof}

% \begin{lemma}\label{lem:nocollision}
% In any round $t>N$, $A(t)=\bar{A}(t)$. 
% \end{lemma}
% \begin{proof}
% Following Algorithm \ref{alg}, all players select different arms in round $t>N$. 
% \end{proof}

In the following, we provide the lemmas that are used in the above analysis. 
We first give an upper bound for the regret that caused by inaccurate preference estimations in Lemma \ref{lem:badevent}. 

% \begin{lemma}\label{lem:badevent}
% \begin{align*}
% \EE{\sum_{t = 1}^T \bOne{\cF(t)} } \le  \frac{NK \pi^2 }{3}\,.
% \end{align*}
% \end{lemma}

% \begin{proof}
% \begin{align}
%     \EE{\sum_{t = 1}^T \bOne{\cF(t)} }
%     = & \EE{\sum_{t = 1}^T \bOne{\exists i \in [N],j\in [K]: \abs{\hat{\mu}_{i,j}(t)-\mu_{i,j}}>\sqrt{\frac{2\log t}{T_{i,j}(t)}}}}  \notag\\
%     \le & \sum_{i \in [N],j\in [K]}\EE{\sum_{t = 1}^T \bOne{\abs{\hat{\mu}_{i,j}(t)-\mu_{i,j}}>\sqrt{\frac{2\log t}{T_{i,j}(t)}}}}  \notag\\
%     = & \sum_{i \in [N],j\in [K]}\sum_{t = 1}^T \sum_{w = 1}^{t} \PP{T_{i,j}(t)=w, \abs{\hat{\mu}_{i,j}(t)-\mu_{i,j}}>\sqrt{\frac{2\log t}{T_{i,j}(t)}} } \notag\\ 
%     = & \sum_{i \in [N],j\in [K]}\sum_{t = 1}^T \sum_{w = 1}^{t}\PP{T_{i,j}(t)=w} \cdot \PP{\abs{\hat{\mu}_{i,j}(t)-\mu_{i,j}}>\sqrt{\frac{2\log t}{T_{i,j}(t)}}\mid T_{i,j}(t)=w } \notag  \\
%     \le & \sum_{i \in [N],j\in [K]}\sum_{t = 1}^T \sum_{w = 1}^{t}\PP{T_{i,j}(t)=w} \cdot 2\exp\bracket{-2\log t}  \label{eq:upper:chernoff}\\
%     \le & \sum_{i \in [N],j\in [K]}\sum_{t = 1}^T \frac{2}{t^2} \notag 
%     ~\le  \frac{NK \pi^2 }{3}  \notag \,,
% \end{align}
% where Eq.\eqref{eq:upper:chernoff} comes from Lemma \ref{lem:chernoff}. 
% \end{proof}

\begin{lemma}\label{lem:badevent}
\begin{align*}
    \PP{\cF} \le 2NK/T \,.
\end{align*}
\end{lemma}

\begin{proof}
\begin{align*}
    \PP{\cF} &= \PP{ \exists 1 \le t\le T, i\in[N], j\in[K]: \abs{ \hat{\mu}_{i,j}(t) -{\mu}_{i,j}} > \sqrt{\frac{ 6\log T}{ T_{i,j}(t)}} } \\
    &\le \sum_{t=1}^T \sum_{i\in [N]}\sum_{j\in [K]} \PP{ \abs{ \hat{\mu}_{i,j}(t) -{\mu}_{i,j}} > \sqrt{\frac{6 \log T}{ T_{i,j}(t)}  } }\\
    &\le \sum_{t=1}^T \sum_{i\in [N]}\sum_{j\in [K]} \sum_{s=1}^{t} \PP{ T_{i,j}(t)=s, \abs{ \hat{\mu}_{i,j}(t) -{\mu}_{i,j}} > \sqrt{\frac{ 6\log T}{ s }  } }\\
    % &\le \sum_{t=1}^T \sum_{i\in [N]}\sum_{j\in [K]} \sum_{s=1}^{t} \PP{  \abs{ \hat{\mu}_{i,j}(t) -{\mu}_{i,j}} > \sqrt{\frac{ \log T}{ s  }} } \\
    &\le \sum_{t=1}^T \sum_{i\in [N]}\sum_{j\in [K]} t\cdot 2 \exp(-3\ln T) \\
    &\le 2NK/T \,,
\end{align*} 
where the second last inequality is due to Lemma \ref{lem:chernoff}. 
\end{proof}

The following Lemma \ref{lem:phase3} shows that each player $p_i$ can successfully find the optimal stable arm $m_i^*$ in at most $N^2$ rounds during the third phase.

\begin{lemma}\label{lem:phase3}
Conditional on $ \urcorner \cF$, 
at most $N^2$ rounds are needed in phase 3 before $\sigma_{i,s} = m^*_i$. And in all of the following rounds, $s$ would not be updated and $p_i$ would always be successfully accepted by $m^*_i$.
\end{lemma}

\begin{proof}
According to Lemma \ref{lem:ucblcb} and Algorithm \ref{alg}, when player $p_i$ enters in phase 3 with $\sigma_i$, we have 
\begin{align*}
    \mu_{i,\sigma_{i,k}} > \mu_{i,\sigma_{i,k+1}}, \text{ for any $k\in[N]$, and } \mu_{i,\sigma_{i,N}} > \mu_{i,\sigma_{i,k}}, \text{ for any $k=N+1,N+2,\ldots, K.$ }
\end{align*}
That's to say, the first $N$ ranked arms in $\sigma_i$ are exactly just the first $N$ ranked arms in the real preference ranking of player $p_i$. 
Further, according to Lemma \ref{lem:phase2}, all players enter in phase 3 simultaneously. 
% with the correct estimations on their preference rankings. 
Combining Lemma \ref{lem:GS:complexity}, GS would stop in at most $N^2$ steps and at most $N$ arms are involved in the process. 
Above all, the procedure of phase 3 is equivalent to the procedure of the offline Gale-Shapley algorithm \citep{gale1962college} with the players' real preference rankings for their top $N$ preferred arms. 
Thus at most $N^2$ rounds are needed before each player $p_i$ successfully finds the optimal stable arm $m^*_i$. 
And once the optimal stable matching is reached, no rejection happens anymore and $s$ would not be updated. Thus each player $p_i$ would always be accepted by $m_{i}^*$ in the following rounds. 
\end{proof}

\begin{lemma}\label{lem:GS:complexity}
    In the offline GS algorithm, at most $N$ arms have been proposed by players before the algorithm stops, thus the player-optimal stable arm of each player must be its first $N$-ranked. 
    And GS would reach player-optimal stability in at most $N^2$ steps.  
\end{lemma}

\begin{proof}
    Based on the offline GS algorithm, once an arm is proposed, it has a temporary player. By contradiction, once $N$ arms have been proposed, it means that $N$ players are occupied. In this case, each player has a partner and the algorithm already stops. Since players propose arms one by one according to their preference ranking, it implies that the player-optimal stable arm of each player must be its first $N$ ranked. 

    And these $N$ arms would reject each player at most once and at least one rejection happens at a step, thus at most $N^2$ steps are required when the algorithm stops. 
\end{proof}

The following Lemma \ref{lem:phase2} analyzes the length of the second phase. 

\begin{lemma}\label{lem:phase2}
Conditional on $\urcorner \cF$, phase 2 will proceed in at most $\ell_{\max}$ sub-phases where 
\begin{align}
    \ell_{\max} = \min\set{\ell: \sum_{\ell'=1}^{\ell} 2^{\ell'} \ge 96K\log T/\Delta^2}\,, \label{eq:def:ellmax}
\end{align}
which implies that $\sum_{\ell'=1}^{\ell_{\max}} 2^{\ell'} \le 192K\log T/\Delta^2$ and $\ell_{\max} = \log \bracket{\log \bracket{192K\log T/\Delta^2} }$  since the sub-phase length grows exponentially. 
And all players will enter in phase 3 simultaneously at the end of sub-phase $\ell_{\max}$. 
\end{lemma}

\begin{proof}
Since players propose to arms based on their distinct indices in a round-robin way, no collision occurs and all players can be successfully accepted at each round in the exploration stage of phase 2. Thus at the end of the sub-phase $\ell_{\max}$ defined in Eq. \eqref{eq:def:ellmax}, it holds that $T_{i,j}\ge 96\log T/\Delta^2$ for any $i\in[N],j\in[K]$. 

According to Lemma \ref{lem:pulltime}, when $T_{i,j}\ge 96\log T/\Delta^2$ for any arm $a_j$, player $p_i$ finds a permutation $\sigma_i$ over arms such that $\lcb_{i,\sigma_{i,k}}>\ucb_{i,\sigma_{i,k+1}}$ for any $k\in[N]$ and $\lcb_{i,\sigma_{i,N}}>\ucb_{i,\sigma_{i,k}}$ for any $k=N+1,N+2,...,K$. 
Thus, at the monitoring round of sub-phase $\ell_{\max}$, each player $p_i$ would propose to the arm with its distinct index. And each player can then observe that $\abs{O_{\ell_{\max}} }= N$. 
Based on this observation, all players would enter in phase 3 simultaneously at the end of sub-phase $\ell_{\max}$. 
\end{proof}

In the following Lemma \ref{lem:ucblcb}, we show that conditional on $\urcorner \cF$, once player $p_i$ observes that the UCB of an arm $a_j$ is smaller than the LCB of another arm $a_{j'}$, we can conclude player $p_i$ truly prefers $a_{j'}$ to $a_j$. 

\begin{lemma}\label{lem:ucblcb}
Conditional on $\urcorner \cF$, $\ucb_{i,j}(t)<\lcb_{i,j'}(t)$ implies $\mu_{i,j}<\mu_{i,j'}$. 
\end{lemma}

\begin{proof}
According to the definition of $\lcb$ and $\ucb$, we have
\begin{align*}
 \lcb_{i,j}(t) =\hat{\mu}_{i,j}(t)-\sqrt{\frac{6\log T}{ T_{i,j}(t)}}  \le \mu_{i,j} \le \hat{\mu}_{i,j}(t)+\sqrt{\frac{6\log T}{ T_{i,j}(t)}} = \ucb_{i,j}(t)\,,
\end{align*}
where two inequalities comes from $\urcorner \cF$. 
Thus if $\ucb_{i,j}(t)<\lcb_{i,j'}(t)$, there would be 
\begin{align*}
  \mu_{i,j} \le  \ucb_{i,j}(t) <\lcb_{i,j'}(t) \le \mu_{i,j'}\,.
\end{align*}
Thus the lemma is proved. 
\end{proof}

Lemma \ref{lem:pulltime} further presents an upper bound for the number of observations required to estimate the preference ranking well. 

\begin{lemma}\label{lem:pulltime}
In round $t$, let $T_i(t) = \min_{j\in[K]}T_{i,j}(t)$ and  $\bar{T}_i = {96\log T}/{\Delta^2}$. Conditional on $\urcorner \cF$, if $T_i(t) > \bar{T}_i$, we have $\lcb_{i,\rho_{i,k}}(t)>\ucb_{i,\rho_{i,k+1}}(t)$ for any $k\in[N]$, and $\lcb_{i,\rho_{i,N}}(t)>\ucb_{i,\rho_{i,k}}(t)$ for any $k=N+1,N+2,\ldots,K$. 
% Thus we can find a permutation $\sigma$ over arms such that $\lcb_{i,\sigma_{j}}>\ucb_{i,\sigma_{j+1}}$ which also is the real preference ranking of player $p_i$. 
\end{lemma} 

\begin{proof}
By contradiction, suppose there exists $k\in [N]$ such that 
$\lcb_{i,\rho_{i,k}}(t)\le \ucb_{i,\rho_{i,k+1}}(t)$ or there exists $k=N+1,N+2,\ldots, K$ such that $\lcb_{i,\rho_{i,N}}(t)\le \ucb_{i,\rho_{i,k}}(t)$. 
Without loss of generality, denote $j$ as the arm in the RHS and $j'$ as the arm in the LHS in above cases. 

According to $\urcorner \cF$ and the definition of $\lcb$ and $\ucb$, we have 
\begin{align*}
 \mu_{i,j'}- 2\sqrt{\frac{6\log T}{{T}_{i}(t)}} \le \lcb_{i,j'}(t) \le \ucb_{i,j}(t) \le \mu_{i,j}+2\sqrt{\frac{6\log T}{{T}_{i}(t)}} \,.
\end{align*}
We can then conclude $\Delta_{i,j,j'} =\mu_{i,j'}-\mu_{i,j} \le 4 \sqrt{\frac{6\log T}{{T}_{i}(t)}}$, which implies that ${T}_{i}(t) \le\frac{96 \log T}{\Delta_{i,j,j'}^2}\le \frac{96 \log T}{\Delta^2}$. This contradicts the fact that $T_i(t) > \bar{T}_i$. 
\end{proof}

\section{Conclusion}\label{sec:conclusion}

In this paper, we study the bandit learning problem in general decentralized matching markets. We design an appropriate way to balance the EE trade-off in the scenario of matching markets. 
Our proposed ETGS algorithm ensures sufficient exploration by the manually designed players' selections and combines it with the exploitation by following the steps of GS with the accurate estimated preference rankings. 
Such a design is shown to be effective and efficient to learn player-optimal stable matching. 
To the best of our knowledge, our result is the first polynomial upper bound for the player-optimal stable regret in general decentralized markets. 
This result also achieves a significant improvement over previous works in the same setting and is also much better than previous results in settings with stronger assumptions or weaker objectives. 

The current algorithm still requires observations on arms' decisions during the monitoring rounds. 
A future direction would be to remove these observations and derive `communication'-free algorithms in general markets that can also achieve player-optimal stable matching.

\bibliography{ref}
\bibliographystyle{named}

\appendix

\section{Technical Lemmas}

\begin{lemma}{(Corollary 5.5 in \citet{lattimore2020bandit})}\label{lem:chernoff}
Assume that $X_1, X_2,\ldots, X_n$ are independent, $\sigma$-subgaussian random variables centered around $\mu$. Then for any $\varepsilon > 0$,
\begin{align*}
    \PP{ \frac{1}{n} \sum_{i=1}^n X_i \ge  \mu + \varepsilon} \le \exp\bracket{-\frac{n\varepsilon^2}{2\sigma^2}}\,, \ \ \ \PP{ \frac{1}{n} \sum_{i=1}^n X_i \le  \mu - \varepsilon} \le \exp\bracket{-\frac{n\varepsilon^2}{2\sigma^2}}\,.
\end{align*}
\end{lemma}

\end{document}